\documentclass[12pt,letterpaper]{elsarticle}
\usepackage[utf8x]{inputenc}
\usepackage{units}
\usepackage{amsthm}
\usepackage{amsmath}
\usepackage{amssymb}
\usepackage{cite}

\makeatletter

\numberwithin{equation}{section}
\numberwithin{figure}{section}
  \theoremstyle{definition}
  \newtheorem{defn}{\protect\definitionname}
\theoremstyle{plain}
\newtheorem{thm}{\protect\theoremname}
\theoremstyle{plain}
\newtheorem{lmm}{\protect\lemmaname}
  \theoremstyle{plain}
  \newtheorem{cor}{\protect\corollaryname}
  \theoremstyle{remark}
  \newtheorem{rem}{\protect\remarkname}
  \theoremstyle{plain}
  \newtheorem{prop}{\protect\propositionname}
 \theoremstyle{definition}
  \newtheorem{example}{\protect\examplename}

\makeatother

  \providecommand{\definitionname}{Definition}
  \providecommand{\examplename}{Example}
  \providecommand{\propositionname}{Proposition}
  \providecommand{\remarkname}{Remark}
\providecommand{\corollaryname}{Corollary}
\providecommand{\theoremname}{Theorem}
\providecommand{\lemmaname}{Lemma}
\begin{document}

\title{Weakly monotone averaging functions}

\author{Tim Wilkin, Gleb Beliakov\\
 School of Information Technology, Deakin University, \\
 221 Burwood Hwy, Burwood 3125, Australia\\
tim.wilkin@deakin.edu.au, gleb@deakin.edu.au }
\begin{abstract}
Monotonicity with respect to all arguments is fundamental to the definition
of aggregation functions. It is also a limiting property that results
in many important non-monotonic averaging functions being excluded
from the theoretical framework. This work proposes a definition
for weakly monotonic averaging functions, studies some  properties
of this class of functions and proves that several families of important
non-monotonic means are actually weakly monotonic averaging functions.
Specifically we provide sufficient conditions for weak monotonicity
of the Lehmer mean and generalised mixture operators. We establish
weak monotonicity of several robust estimators of location and conditions
for weak monotonicity of a large class of penalty-based aggregation
functions. These results permit a proof of the weak monotonicity of
the class of spatial-tonal filters that include important members
such as the bilateral filter and anisotropic diffusion. Our concept
of weak monotonicity provides a sound theoretical and practical basis
by which (monotone) aggregation functions and non-monotone averaging functions
can be related within the same framework, allowing us to bridge the
gap between these previously disparate areas of research.
\end{abstract}
\maketitle

\newpage

\section{Introduction}

\label{sec:Introduction}

The aggregation of several input variables into a single representative
output arises naturally as a problem in many practical applications
and domains. The research effort has been disseminated throughout
various fields including economics, computer science, mathematics
and engineering, with the subsequent mathematical formulation of aggregation
problems having coalesced into a significant body of knowledge concerning
aggregation functions. A wide range of aggregation functions are presented
in the literature, including the weighted quasi-arithmetic means,
ordered weighted averages, triangular norms and co-norms, Choquet
and Sugeno integrals and many more. Several recent books provide a
comprehensive overview of this field of study (\citet{Beliakov2007_book,Grabisch2009_book, Torra2007_book}).

Aggregation functions are commonly used within fuzzy logic, where
logical connectives are typically modeled using triangular norms and
triangular co-norms. Beyond this field  the averaging functions
- more commonly known as \emph{means} - that are frequently applied
in decision problems, statistical analysis and in image and signal
processing. Means have been an important tool and topic of study for
over two millennia, with examples such as the arithmetic, geometric
and harmonic means known to the Greeks (\citet{Rubin1968}). Each
of these means shares a fundamental property with the broader class
of aggregation functions; that of monotonicity with respect to all
arguments (\citet{Beliakov2007_book,Grabisch2009_book,Torra2007_book}).
There are though many means of significant practical and theoretical
importance that are non-monotonic and hence not classified as aggregation
functions. For example, a non-monotonic average of pixel intensities
within an image subset is used to perform tasks such as image reduction
(\citet{Wilkin2013a}), filtering (\citet{Boomgaard2002,Sylvain2008})
or smoothing (\citet{Barash2004}). Within statistics, robust estimators
of location are used to estimate the central tendency of a data set
and the mode, an average possibly known to the Greeks (\citet{Rubin1971}),
is a classic example of a non-monotonic average.

Monotonicity with respect to all arguments has an important interpretation
in decision making problems: an increase in one criterion should not
lead to the decrease of the overall score or utility. However, in
image processing an increase in only one pixel value above its neighbours
may be due to noise or corruption and should not necessarily increase
the intensity value that represents that region. Accordingly, the
averaging functions used in such applications do not fit within the
established theories regarding aggregation functions and are typically
dealt with only from the signal processing perspective.

There are also many non-monotonic means appearing in the literature,
with the mode, Gini means, Lehmer means, Bajraktarevic means (\citet{Beliakov2007_book,Bullen2003_book})
and mixture functions (\citet{Ribeiro2003,Marques2003_FSS}) being
particularly well known cases. Ideally we would like a formal framework
for averaging functions that encompasses non-monotonic means and places
them in context with existing monotonic aggregation functions, enabling
us to better understand the relationships within this broad class
of functions. In so doing we are then able to broaden our understanding
of non-monotonic averaging as an aggregation problem.

We achieve this aim herein by relaxing the monotonicity requirement
for averaging aggregation functions and propose a new definition that encompasses
many non-monotonic averaging functions. We justify this approach by
the following interpretation of averaging: while we accept that an
increase in one input, or coordinate, may lead to a decrease of the
aggregate value, we argue that the same increase coincident in all
inputs should only lead to an increase of the aggregate value. This
is akin to the property of shift-invariance, which along with homogeneity
is one of the basic requirements of the non-monotonic location estimators
(\citet{Rousseeuw1987}). We do not impose shift-invariance though,
as that would severely limit the range of averaging functions that
fall under our definition averaging functions (for instance, the
only shift invariant quasi-arithmetic means are weighted arithmetic
means). Rather we consider the property of directional monotonicity
in the direction of the vector $(1,1,\ldots,1)$, which is obviously
implied by shift-invariance as well as by the standard definition
of monotonicity. We call this property \emph{weak monotonicity} within
the context of aggregation functions and we investigate it herein.

The remainder of this article is structured as follows. In Section
\ref{sec:preliminaries} we provide the necessary mathematical foundations
that underpin the subsequent material. Section \ref{sec:weak_monotonicty}
provides the main definitions and presents various properties of weakly
monotone aggregation functions. Within Section \ref{sec:wm_examples}
we examine several non-monotonic means and prove that they are, in
fact, weakly monotonic. In Section \ref{sec:conclusions} we draw
our conclusions and discuss future research directions arising as
a result of this investigation.

\section{Preliminaries}
\label{sec:preliminaries}

\subsection{Aggregation functions}

In this article we make use of the following notations and assumptions.
Without loss of generality we assume that the domain of interest is
any closed, non-empty interval $\mathbb{I}\subseteq\bar{\mathbb{R}}=[-\infty,\infty]$
and that tuples in $\mathbb{I}^{n}$ are defined as $\mathbf{x}=(x_{i,n}|n\in\mathbb{N},i\in\{1,...,n\})$.
We write $x_{i}$ as the shorthand for $x_{i,n}$ such that it is
implicit that $i\in\{1,...,n\}$. Furthermore, $\mathbb{I}^{n}$ is
ordered such that for $\mathbf{x},\mathbf{y}\in\mathbb{I}^{n}$, $\mathbf{x}\le\mathbf{y}$
implies that each component of $\mathbf{x}$ is no greater than the
corresponding component of $\mathbf{y}$.
Unless otherwise stated, a constant vector given as $\mathbf{a}$
is taken to mean $\mathbf{a}=a({\underbrace{1,1,\ldots,1}_\text{$n$ times})=a\mathbf{1}}$,
where $a \in \mathbb{R}$ is a constant and $n$ is implicit within the context of use.

The vector $\mathbf{x}_{\nearrow}$ denotes the result of permuting
the vector $\mathbf{x}$ such that its components are in non-decreasing
order, that is, $\mathbf{x}_{\nearrow}=\mathbf{x}_{\sigma}$, where
$\sigma$ is the permutation such that $x_{\sigma(1)}\le x_{\sigma(2)}\le\ldots\le x_{\sigma(n)}$.
Similarly, the vector $\mathbf{x}_{\searrow}$ denotes the result
of permuting $\mathbf{x}$ such that $x_{\sigma(1)}\ge x_{\sigma(2)}\ge\ldots\ge x_{\sigma(n)}$.
We will make use of the common shorthand notation for a sorted vector,
being $\mathbf{x_{()}}=(x_{(1)},x_{(2)},\ldots,x_{(n)})$. In such
cases the ordering will be stated explicitly and then $x_{(k)}$ represents
the $k$th largest or smallest element of $\mathbf{x}$ accordingly. 

Consider now the following definitions:
\begin{defn}
\label{def:monotonic_function}\emph{}{A function $F:\mathbb{I}^{n}\rightarrow\bar{\mathbb{R}}$
is }\textbf{monotonic}\emph{}{ (non-decreasing) if and only if, $\forall\mathbf{x},\mathbf{y}\in\mathbb{I}^{n},\mathbf{x}\le\mathbf{y}$
then $F(\mathbf{x})\le F(\mathbf{y})$.}
\end{defn}

\begin{defn}
\label{def:aggregation_function}\emph{}{A function} $F:\mathbb{I}^{n}\rightarrow\mathbb{I}$
\emph{}{is an }\textbf{aggregation function} \emph{}{in $\mathbb{I}^{n}$
if and only if $F$ is monotonic non-decreasing in $\mathbb{I}$ and
$F(\mathbf{a})=a$, $F(\mathbf{b})=b$, with $\mathbb{I}^{n}=[a,b]^{n}$.}
\end{defn}

\begin{defn}
\label{def:idempotent_function}\emph{}{A function $F$ is called }\textbf{idempotent}\emph{}{
if for every input $\mathbf{x}=(t,\, t,\,...\,,\, t),\, t\in\mathbb{I}$
the output is $F(\mathbf{x})=t$.}
\end{defn}
The functions of most interest in this article are those that have
averaging behaviour.
\begin{defn}
\textbf{\label{def:averaging_function}}\emph{}{A function $F$ has }\textbf{averaging behaviour}{ (or is averaging) if for every
$\mathbf{x}$ it is bounded by} $\min(\mathbf{x})\le F(\mathbf{x})\le\max(\mathbf{x}).$
\end{defn}
Aggregation functions that have averaging behaviour
are idempotent, whereas idempotency and monotoicity  imply averaging behaviour.
 
\begin{defn}
\label{def:internal_function}\emph{}{A function is called }\textbf{internal}{
if its value coincides with one of the arguments.}
\end{defn}

Of particular relevance is the notion of shift-invariance 
\citet{Calvo2002_book,Lazaro2004_FSS} 
(which is also called difference scale invariance \citet{Grabisch2009_book}).
 A constant change in every input should result in a corresponding change of the
output.

\begin{defn}
{\label{def:shift_invariant}A function $F:\mathbb{I}^{n}\rightarrow\mathbb{I}$
is} \textbf{shift-invariant} \emph{(stable for translations) if} $F(\mathbf{x}+a\mathbf{1})=F(\mathbf{x})+a$
whenever $\mathbf{x},\mathbf{x}+a\mathbf{1}\in\mathbb{I}^{n}$.
\end{defn}

\begin{defn}
\emph{}{\label{def:homogeneous_function}A function $F$ is} \textbf{homogeneous}
\emph{}{(with degree one) if} $F(a\mathbf{x})=aF(\mathbf{x})$ \emph{}for
all $a\mathbf{x}\in\mathbb{I}^{n}$.
\end{defn}
Aggregation functions that are shift-invariant and homogeneous are
known as \emph{linear aggregation functions}.
 The canonical example of a linear
aggregation function is the arithmetic mean.

\subsection{Means}

The term \emph{mean} is used synonymously with averaging aggregation
functions. Chisini's definition of a mean as an average states that
the mean of $n$ independent variables $(x_{1},...,x_{n})$, with
respect to a function $F$, is a value $M$ for which replacement
of each value $x_{i}$ in the input by $M$, results in the output
$M$ (\citet{Chisini1929}, stated in \citet{Grabisch2009_book}).
I.e.,

\[
F(x_{1},...,x_{n})=F(M,...,M)=M
\]

As was noted by de Finetti (\citet{deFinetti1931}, stated in \citet{Grabisch2009_book}),
Chisini's definition does not necessarily satisfy Cauchy's requirement
that a mean be an internal value (\citet{Cauchy1821}). However, by
assuming that $F$ is a non-decreasing, idempotent function, then
existence, uniqueness and internality of $M$ are restored to Chisini's
definition. Gini ( \citet{Gini1958_book}, p.64), writes that an average of several quantities is a value obtained as a result of a certain procedure, which equals to either one of the input quantities, or a new value that lies in between the smallest and the largest input.
The requirement that $F$ be non-decreasing is too strict
given the aims of this article and as such, following many authors
(e.g., \citet{Gini1958_book,Bullen2003_book}), we take the definition of a mean
to be any averaging (and hence idempotent) function.
\begin{defn}
\label{def:mean}{A function} $M:\mathbb{I}^{n}\rightarrow\mathbb{I}$
{is called a} \textbf{mean} if and only if it is averaging.
\end{defn}

The basic examples of (monotonic) means found within the
literature include
 weighted arithmetic mean,
 weighted quasi-arithmetic mean,
 ordered weighted average (OWA),
 order statistic $S_{k}(\mathbf{x})=x_{(k)}$, and
the median. Less known examples include Choquet and Sugeno integrals and their special cases;
the logarithmic means, Heronian means, Bonferroni means and others
\citet{Bullen2003_book, Grabisch2009_book}.

In continuing, we wish to consider a broader class of means to include
those that are not necessarily monotonic. A classic example is the
mode, being the most frequent input, which is routinely used in statistics. \footnote{In general the mode is multivalued, so in order to make it a single-variate function, a convention is needed to select one of the multiple outputs, e.g. the smallest. }
The mode is not monotonic as the following example shows. Taking the
vectors $\mathbf{x}=(1,1,2,2,3,3,3),\mathbf{y}=(1,1,0,0,0,0,0)$,
and $\mathbf{z}=(1,1,1,1,1,1,1)$, then $Mode(\mathbf{x})=3,Mode(\mathbf{x}+\mathbf{y})=2$
and $Mode(\mathbf{x}+\mathbf{z})=4$.

An important class of means that are not always monotonic are those
expressed by the Mean of Bajraktarevic, which is a generalisation
of the weighted quasi-arithmetic means.
\begin{defn}
\textbf{\label{def:Mean-of-Bajraktarevic}Mean of Bajraktarevic}.
\emph{}Let $\mathbf{w}(t)=(w_{1}(t),...,w_{n}(t))$ be a vector of
weight functions $w_{i}:\mathbb{I}\rightarrow[0,\infty)$, and let
$g:\mathbb{I}\rightarrow\mathbb{\bar{R}}$ be a strictly monotonic
function. The mean of Bajraktarevic is the function
\begin{equation}
M_{\mathbf{w}}^{g}(\mathbf{x})=g^{-1}\left(\frac{{\displaystyle \sum_{i=1}^{n}w_{i}(x_{i})g(x_{i})}}{{\displaystyle \sum_{i=1}^{n}}w_{i}(x_{i})}\right)\label{eq:bajraktarevic_mean}
\end{equation}

\end{defn}
When $g(x_{i})=x_{i}$, and all weight functions are the same, the Bajraktarevic mean is called a \emph{mixture
function} (or\emph{ mixture operator}) and is given by

\begin{equation}
M_{w}(\mathbf{x})=\frac{{\displaystyle \sum_{i=1}^{n}}w(x_{i})x_{i}}{{\displaystyle \sum_{i=1}^{n}}w(x_{i})}\label{eq:mixture_function}
\end{equation}
For the case where the weight functions are are distinct $w_{i}(x_{i})$, the operator
$M_{\mathbf{w}}(\mathbf{x})$ is a \emph{generalised mixture function}.
A particularly interesting sub-class of Bajraktarevic means are Gini
means, obtained by setting $w_{i}(t)=w_{i}t^{q}$ and $g(t)=t^{p}$
when $p\ne0$, or $g(t)=\log(t)$ if $p=0$.

\begin{equation}
G_w(\mathbf{x};p,q)=\left(\frac{{\displaystyle \sum_{i=1}^{n}w_{i}x_i^{p+q}}}{{\displaystyle \sum_{i=1}^{n}}w_{i}x_i^{q}}\right)^{\frac{1}{p}}\label{eq:Gini_mean}
\end{equation}

Gini means generalise the (weighted) power means (for $q=0$) and
hence include the minimum, maximum and the arithmetic mean as special
cases. Another special case of the Gini mean is the Lehmer, or counter-harmonic
mean, obtained when $p=1$. The contra-harmonic mean is the Lehmer
mean with $q=1$. We will investigate the Lehmer mean and its properties
further in Section \ref{sec:wm_examples}.

\subsection{Penalty based functions}

In \citet{Calvo2010a} it was demonstrated that averaging aggregation
functions can be expressed as the solution of a minimisation problem
of the form
\begin{equation}
F(\mathbf{x})=\arg\underset{y}{\min}\:\mathcal{P}(\mathbf{x},y)\label{eq:penalty_based_function}
\end{equation}
where $\mathcal{P}(\mathbf{x},y)$ is a penalty function satisfying
the following definition:
\begin{defn}
\textbf{\label{def_Penalty-function}Penalty function}. The function
$\mathcal{P}:\mathbb{I}^{n+1}\rightarrow\mathbb{R}$ is a penalty
function if and only if it satisfies:
\begin{enumerate}
\item $\mathcal{P}(\mathbf{x},y)\ge c\quad\forall\,\mathbf{x}\in\mathbb{I}^{n},\, y\in\mathbb{I}$;
\item $\mathcal{P}(\mathbf{x},y)=c$ if and only if all $x_{i}=y$; and,
\item $\mathcal{P}(\mathbf{x},y)$ is quasi-convex in $y$ for any $\mathbf{x}$,
\end{enumerate}
for some constant $c\in\mathbb{R}$ and any closed, non-empty interval
$\mathbb{I}$.
\end{defn}

 A function $P$ is quasi-convex if  all its sublevel sets are convex, that is $S_\alpha(P) = \{x|P(x) \leq \alpha\}$ are convex sets for all $\alpha$, see \citet{Rockafellar1970_book}.
The first two conditions ensure that $\mathcal{P}$
has a strict minimum and that a consensus of inputs ensures minimum
penalty, providing idempotence of $F(\mathbf{x})$. The third condition
implies a unique minimum (but possibly many minimisers that form a convex set). Since multiplication
by, or addition of a constant to $\mathcal{P}$ will not change the
minimisation, $\mathcal{P}$ may be shifted (if desired) so that $c=0$.
One can think of $\mathcal{P}$ as describing the dissimilarity or
disagreement between the inputs $\mathbf{x}$ and the value $y$.
It follows that $F$ is a function that minimises the chosen dissimilarity.
It is not necessary to explicitly state $F$, provided a suitable
penalty function is given and the optimisation problem solvable. Subsequently
it is sufficient to solve (\ref{eq:penalty_based_function}) to obtain
the aggregate $\mu=F(\mathbf{x})$.

Non-monotonic averaging functions can also be represented by a penalty
function. For penalty-based functions we have
the following results due to \citet{Calvo2010a}.
\begin{thm}
\label{thm: averaging-idempotent-penalty}Any idempotent function
$F:\mathbb{I}^{n}\rightarrow\mathbb{I}$ can be represented as a penalty
based function $\mathcal{P}:\mathbb{I}^{n+1}\rightarrow\mathbb{I}$
such that
\[
F(\mathbf{x})=\arg{\displaystyle \min_{y}}\mathcal{P}(\mathbf{x},y).
\]
\end{thm}
\begin{cor}
Any averaging function can be expressed as a penalty based function.
\end{cor}

\label{thm:penalty_mixture_function}
As mentioned in \citet{Mesiar2008}, mixture functions can be written
as a penalty function with
\[
\mathcal{P}(\mathbf{x},y)={\displaystyle \sum_{i=1}^{n}w}(x_{i})(x_{i}-y)^{2}.
\]
Clearly the necessary condition of the minimum is  $$\mathcal{P}_{y}(\mathbf{x},y)=-2\left({\displaystyle \sum_{i=1}^{n}}w(x_{i})x_{i}-y{\displaystyle \sum_{i=1}^{n}}w(x_{i})\right)=0.$$
Hence $\mathcal{P}(\mathbf{x},y)$ defines a mixture function. A representation of a function as a penalty based function sometimes can simplify technical proofs, as we shall see later in the paper.

It is apparent given the examples presented that many means are non-monotonic
and thus not aggregation functions according to Definition \ref{def:aggregation_function}.
In the next section we introduce weak monotonicity and consider some
properties of weakly monotonic averaging functions. We subsequently investigate
several important examples and show that they are indeed weakly monotonic
functions, allowing us to place them in a new framework with existing
averaging aggregation functions.

\section{Weak monotonicity}

\label{sec:weak_monotonicty}

\subsection{Main definition}

As mentioned in Section \ref{sec:Introduction} we are motivated by
two important issues. The first one is that there exist many means that
are not generally monotonic
and hence not aggregation functions, while the second one is that there
are many practical applications in which non-monotonic means have
shown to provide good aggregate values commensurate with the objectives
of the aggregation. To encapsulate these non-monotonic means within
the framework of aggregation functions we aim to relax the monotonicity
condition and present the class of \emph{weakly monotonic averaging
functions}. The definition of weak monotonicity provided herein is
prompted by applications and intuition, which suggests that it is
reasonable to expect that a representative value of the inputs does
not decrease if all the inputs are increased by the same amount (or
shifted uniformly) as the relative positions of the inputs are not changed. A formal definition that
conveys this property is as follows.
\begin{defn}
\textbf{\label{def:weak_monotonicity}}\emph{A function $f$ is called
}\textbf{weakly monotonic}\emph{ non-decreasing (or directionally
monotonic) if $F(\mathbf{x}+a\mathbf{1})\geq F(\mathbf{x})$ for any
$a>0,\ ({\underbrace{1,1,\ldots,1}_{n\, \text{times}}})$,
such that $\mathbf{x},\mathbf{x}+a\mathbf{1}\in\mathbb{I}^{n}$.}\end{defn}
\begin{rem}
If $F$ is directionally differentiable in its domain then weak monotonicity
is equivalent to non-negativity of the directional derivative $\textrm{D}_{\mathbf{1}}(F)(\mathbf{x})\geq0$.
\end{rem}

\begin{rem}
Evidently monotonicity implies weak monotonicity, hence all aggregation
functions are weakly monotonic. By Definition \ref{def:shift_invariant}
all shift-invariant functions are also weakly monotonic. It is self
evident that weakly monotonic non-decreasing functions form a cone
in the linear vector space of weakly monotonic (increasing or decreasing)
functions.
\end{rem}

\subsection{Properties}

Let us establish some useful properties of weakly monotonic averages.
Consider the function $F:\mathbb{I}^{n}\rightarrow\mathbb{I}$ formed
by the composition $F(\mathbf{x})=A(B_{1}(\mathbf{x}),B_{2}(\mathbf{x}))$,
where $A,B_{1}$ and $B_{2}$ are means.
\begin{prop}
If $A$ is monotonic and $B_{1},B_{2}$ are weakly monotonic, then $F$
is weakly monotonic.\end{prop}
\begin{proof}
By weak monotonicity $B_{i}(\mathbf{x}+a\mathbf{1})\ge B_{i}(\mathbf{x})$
implies that $\exists\delta_{i}\ge0$ such that $B_{i}(\mathbf{x}+a\mathbf{1})=B_{i}(\mathbf{x})+\delta_{i}$,
with $\mathbf{x},\mathbf{x}+a\mathbf{1}\in\mathbb{I}^{n}$. Thus $F(\mathbf{x}+a\mathbf{1})=A(b_{1}+\delta_{1},b_{2}+\delta_{2})$,
where $b_{i}=B_{i}(\mathbf{x})$. The monotonicity of $A$ ensures
that $A(b_{1}+\delta_{1},b_{2}+\delta_{2})\nless A(b_{1},b_{2})$
and hence $F(\mathbf{x}+a\mathbf{1})\ge F(\mathbf{x})$ and $F$ is
weakly monotonic.
\end{proof}
By trivial extension, since all monotonic functions are also weakly
monotonic, then if either of $B_{1}$ or $B_{2}$ is monotonic, then
$F$ is again weakly monotonic.
\begin{prop}
If $A$ is weakly monotonic and $B_{1},B_{2}$ are shift invariant,
then $F$ is weakly monotonic.\end{prop}
\begin{proof}
Shift invariance implies that $\forall a:B_{i}(\mathbf{x}+a\mathbf{1})=B_{i}(\mathbf{x})+a$,
with $\mathbf{x},\mathbf{x}+a\mathbf{1}\in\mathbb{I}^{n}$. Thus $F(\mathbf{x}+a\mathbf{1})=A(b_{1}+a,b_{2}+a)$,
where $b_{i}=B_{i}(\mathbf{x})$. The weak monotonicity of $A$ ensures
that $A(b_{1}+a,b_{2}+a)\nless A(b_{1},b_{2})$ and hence $F(\mathbf{x}+a\mathbf{1})\ge F(\mathbf{x})$
and $F$ is weakly monotonic.
\end{proof}
Consider functions of the form $\varphi(\mathbf{x})=(\varphi(x_{1}),\varphi(x_{2}),...,\varphi(x_{n}))$.
\begin{prop}
If $A$ is weakly monotonic and $\varphi(x)$ is a linear function
then the $\varphi-$transform $A_\varphi(\mathbf{x})=F(\mathbf{x})=\varphi^{-1}\left(A(\varphi(\mathbf{x}))\right)$
is weakly monotonic.\end{prop}
\begin{proof}
$\varphi(x)=\alpha x+\beta$ and hence $\varphi(x+a)=\alpha(x+a)+\beta=\alpha x+\beta+\alpha a=\varphi(x)+c$.
Hence
\begin{alignat*}{1}
F(\mathbf{x}+a\mathbf{1}) & =\varphi^{-1}\left(A(\varphi(x_{1}+a),...,\varphi(x_{n}+a))\right)\\
 & =\varphi^{-1}\left(A\left(\varphi(\mathbf{x})+c\mathbf{1}\right)\right)\\
 & =\frac{A\left(\varphi(\mathbf{x})+c\mathbf{1}\right)-\beta}{\alpha}\\
 & \ge\frac{A\left(\varphi(\mathbf{x})\right)-\beta}{\alpha}\\
 & =\varphi^{-1}\left(A(\varphi(\mathbf{x})\right))
\end{alignat*}
by weak monotonicity of $A$. Hence $F(\mathbf{x}+a\mathbf{1})\ge F(\mathbf{x})$
and $F$ is weakly monotonic.
\end{proof}

Note that unlike in the case of standard monotonicity, a nonlinear $\varphi$-transform does not always preserve weak monotonicity.
\begin{cor}
The dual $A^{d}$ of a weakly monotonic function $A$ is weakly monotonic
under standard negation.
\end{cor}

The following result is relevant to an application of weakly monotonic averages in image processing, discussed in Section \ref{sec:image}.

\begin{thm}
\label{thm:shift_invariant_penalty} Let $f:\mathbb{I}^{n}\rightarrow\mathbb{I}$ be a shift invariant function, and $g$ be a function. Let $F$ be a penalty based averaging function with the
penalty $\mathcal{P}$ depending on the  terms $g\left(x_{i}-f(\mathbf{x})\right)(x_{i}-y)^{2}$. Then $F$ is shift-invariant and hence weakly monotonic.\end{thm}
\begin{proof}
Let
\[
\mu={\displaystyle \arg \min_{y} \mathcal P \left(g\left(x_{1}-f(\mathbf{x})\right)(x_{1}-y)^{2},...,g\left(x_{n}-f(\mathbf{x})\right)(x_{n}-y)^{2}\right)}.
\]
Then
\begin{alignat*}{1}
\arg \min_{y}\mathcal{P}(\mathbf{x}+a\mathbf{1},y) & =\arg \min_{y} \mathcal P\left(g\left(x_{1}+a-f(\mathbf{x}+a\mathbf{1})\right)(x_{1}+a-y)^{2},...\right.\\
 & \qquad\qquad...,\left.g\left(x_{n}+a-f(\mathbf{x}+a\mathbf{1})\right)(x_{n}+a-y)^{2}\right)=\\
\end{alignat*}
(by shift invariance)
\begin{alignat*}{1}
&=\arg \min_{y} \mathcal P\left(g\left(x_{1}-f(\mathbf{x})\right)(x_{1}+a-y)^{2},\ldots,
g\left(x_{n}-f(\mathbf{x})\right)(x_{n}+a-y)^{2}\right) \\
 &=\mu +a.
\end{alignat*}
\end{proof}

\begin{rem}
\label{remark:penalty_dissimilarity_function}Indeed we need not restrict
ourselves to penalty functions with terms depending on $(x_{i}-y)^{2}$.
Functions $D$ that depend on the differences $x_i-y$ with the minimum $D(0)$
 will satisfy the above proof and satisfy
the conditions on $\mathcal{P}$ with regards to the existence of
solutions to \eqref{eq:penalty_based_function}. In particular, Huber type functions used in robust regression can replace the squares of the differences. 
\end{rem}


\subsection{Counter-cases}

For the $\varphi-$transform, if $\varphi$ is nonlinear then $F$
may or may not be weakly monotonic for all $\mathbf{x}$, which can
be observed by example.
\begin{example}
Take $\mathbf{x}=(1,8,16,35,47.9)$ and $\varphi(t)=\sqrt{t}$, then
$\varphi(\mathbf{x})=(1,2\sqrt{2},4,\sqrt{35},\sqrt{47.9})$ and $\varphi(\mathbf{x}+\mathbf{1})=(\sqrt{2},3,\sqrt{17},6,\sqrt{48.9})$.
If $A$ is the shorth (we prove the weak monotonicity of the shorth
in Section \ref{sec:wm_examples}) then $A(\varphi(\mathbf{x}))=5.61$
and $A(\varphi(\mathbf{x}+\mathbf{1}))=2.84$. As $\varphi^{-1}(t)=t^{2}$
clearly $5.61^{2}>2.84^{2}$ and $F=A_\varphi$ is not weakly monotonic.
\end{example}

Internal means are not necessarily weakly monotonic, as illustrated
by the following example.
\begin{example}
Take $\mathbf{x}=(x_{1,}x_{2})\in[0,1]^{2}$ and
\end{example}
\[
F(\mathbf{x})=\begin{cases}
\min(\mathbf{x}) & \textrm{if}\ x_{1}+x_{2}\ge1\\
\max(\mathbf{x}) & \textrm{otherwise}
\end{cases}
\]
which is internal with values in the set $\{\min(\mathbf{x}),\max(\mathbf{x})\}$.
Consider the points $\mathbf{x}=(\nicefrac{1}{4},0)$ and $\mathbf{y}=(\nicefrac{3}{4},0)$,
then $F(\mathbf{x})=\nicefrac{1}{4}$ and $F(\mathbf{y})=\nicefrac{3}{4}$.
It follows that $F(\mathbf{x}+\nicefrac{1}{4}\mathbf{1})=\nicefrac{1}{2}>F(\mathbf{x})$,
however $F(\mathbf{y}+\nicefrac{1}{4}\mathbf{1})=\nicefrac{1}{4}<F(\mathbf{y})$.
Hence this $F$ is not weakly monotonic for all $\mathbf{x}\in\mathbb{I}^{2}$.

\section{Examples of weakly monotonic means}

\label{sec:wm_examples}

In this section we look at several examples of weakly monotonic, but
not necessarily monotonic averaging functions. We begin by considering
several of the robust estimators of location, then move on to mixture
functions and some interesting cases of means from the literature. While some of the examples involve shift-invariant functions, many of their nonlinear  $\varphi$-transforms yield proper weakly monotonic functions.

The functions presented below are defined through penalties that are not quasi-convex, therefore we need to drop the condition that $\mathcal P$ is quasi-convex from Definition \ref{def_Penalty-function}.

\begin{defn}
\textbf{\label{def_quasiPenalty-function}Quasi-penalty function}. The function
$\mathcal{P}:\mathbb{I}^{n+1}\rightarrow\mathbb{R}$ is a quasi-penalty
function if and only if it satisfies:
\begin{enumerate}
\item $\mathcal{P}(\mathbf{x},y)\ge c\quad\forall\,\mathbf{x}\in\mathbb{I}^{n},\, y\in\mathbb{I}$;
\item $\mathcal{P}(\mathbf{x},y)=c$ if and only if all $x_{i}=y$; and,
\item $\mathcal{P}(\mathbf{x},y)$ is lower semi-continuous in $y$ for any $\mathbf{x}$,
\end{enumerate}
for some constant $c\in\mathbb{R}$ and any closed, non-empty interval
$\mathbb{I}$.
\end{defn}

Note that the third condition ensures the existence of the minimum and a non-empty set of minimisers. In the case where the set of minimisers of $\mathcal{P}$ is not an interval, we need to adopt a reasonable rule for selecting the value of the penalty-based function $F$. We suggest stating in advance that in such cases we choose the infimum of the set of minimisers of $\mathcal{P}$. From now one $\mathcal{P}$ will refer to quasi-penalty functions.

\subsection{Estimators of Location}

Perhaps the most widely used estimator of location is the mode, being
the most frequent input.
\begin{example}
\textbf{Mode}: The mode is the minimiser of the (quasi)penalty function
\[
\mathcal{P}(\mathbf{x},y)={\displaystyle \sum_{i=1}^{n}p(x_{i},y)}\qquad\textrm{where}\quad p(x_{i},y)=\begin{cases}
0 & x_{i}=y\\
1 & \textrm{otherwise}
\end{cases}.
\]
It follows that $F(\mathbf{x}+a\mathbf{1})=\arg{\displaystyle \min_{y}}\mathcal{P}(\mathbf{x}+a\mathbf{1},y)={\displaystyle \arg{\displaystyle \min_{y}}\sum_{i=1}^{n}p(x_{i}+a,y)}$,
which is minimised for the value $y=F(\mathbf{x})+a$. Hence, $F(\mathbf{x}+a\mathbf{1})=F(\mathbf{x})+a$
and thus the mode is shift invariant. By Definition \ref{def:shift_invariant}
the mode is weakly monotonic.
\end{example}

\begin{rem}
Note that the mode may not be uniquely defined, e.g., the mode of $(1,1,2,2,3,4,5)$, in which case we use a suitable convention. The quasi-penalty $\mathcal{P}$ associated with the mode is not quasi-convex, and as such it may have several minimisers. A convention is needed as to which minimiser is selected, e.g., the smallest or the largest. Other examples of non-monotonic means that follow also involve quasi-penalties, and the same convention as for the mode is adopted. Then also discrete scales can be considered, compare to, e.g., the paper of \citet{Kolesarova2007_IS}.
\end{rem}

The Least Trimmed Squares estimator (\citet{Rousseeuw1987}) rejects
up to $50\%$ of the data values as outliers and minimises the squared
residual using the remaining data.
\begin{example}
\textbf{Least Trimmed Squares (LTS):} The LTS uses the (quasi)penalty function
\[
\mathcal{P}(\mathbf{x},y)=\sum_{i=1}^{h}r_{(i)}^{2}
\]
where $r_{(i)}=S_{i}(\mathbf{r})$ is the $i$th order statistic of $\mathbf{r}$, $r_{k}=x_{k}-y$ and $h=\left\lfloor \frac{n}{2}\right\rfloor +1$.
If $\sigma$ is the order permutation of $\{1,...,n\}$ such that
$\mathbf{r_{\sigma}=r_{\nearrow}}$, then the minima of $\mathcal{P}$
occur when $P_{y}=-2{\displaystyle \sum_{i=1}^{h}(x_{\sigma(i)}-y)}=0$,
which implies that the minimum value is $\mu=\frac{1}{h}{\displaystyle \sum_{i=1}^{h}}x_{\sigma(i)}$.
Since $S_{k}(\mathbf{x})$ is shift invariant then $S_{i}(\mathbf{r}+a\mathbf{1})=r_{\sigma(i)}+a$
and thus
\[
\mathcal{P}(\mathbf{x}+a\mathbf{1},y)=\sum_{i=1}^{h}v_{\sigma(i)}^{2}
\]
 where $v_{k}=((x_{k}+a)-y)$. It follows that the value $y$ that
minimises $\mathcal{P}(\mathbf{x}+a\mathbf{1},y)$ is $y=\mu+a$,
hence the LTS is shift invariant and thus weakly monotonic.
\end{example}
The remaining estimators of location presented compute their value
using the shortest contiguous sub-sample of $\mathbf{x}$ containing
at least half of the values. The candidate sub-samples are the sets
$X_{k}=\{x_{j}:j\in\{k,k+1,...,k+\left\lfloor \frac{n}{2}\right\rfloor \},\ k=1,...,\left\lfloor \frac{n+1}{2}\right\rfloor $.
The length of each set is taken as $\left\Vert X_{k}\right\Vert =\left|x_{k+\left\lfloor \frac{n}{2}\right\rfloor }-x_{k}\right|$
and thus the index of the shortest sub-sample is
\[
k^{*}=\arg{\displaystyle \min_{i}}\left\Vert X_{i}\right\Vert ,\ i=1,...,\left\lfloor \frac{n+1}{2}\right\rfloor .
\]
Under the translation $\bar{\mathbf{x}}=\mathbf{x}+a\mathbf{1}$ the
length of each sub-sample is unaltered since $\left\Vert \bar{X}_{k}\right\Vert =\left|\bar{x}_{k+\left\lfloor \frac{n}{2}\right\rfloor }-\bar{x}_{k}\right|=\left|(x_{k+\left\lfloor \frac{n}{2}\right\rfloor }+a)-(x_{k}+a)\right|=\left|x_{k+\left\lfloor \frac{n}{2}\right\rfloor }-x_{k}\right|=\left\Vert X_{k}\right\Vert $
and thus $k^{*}$ remains the same.

Consider now the Least Median of Squares estimator (\citet{Rousseeuw1984}),
which is the midpoint of $X_{k^{*}}$.
\begin{example}
\textbf{Least Median of Squares (LMS):} The LMS can be computed by
minimisation of the (quasi)penalty function
\[
\mathcal{P}(\mathbf{x},y)=median\left\{ (x_{i}-y)^{2}\left|y\in\mathbb{I},\ x_{i}\in X_{k^{*}}\right.\right\}
\]
The value $y$ minimises the penalty $\mathcal{P}(\mathbf{x}+a\mathbf{1},y)$,
given by
\[
{\displaystyle \min_{y}}\mathcal{P}(\mathbf{x}+a\mathbf{1},y)={\displaystyle \min_{y}}\ median\left\{ (x_{j}+a-y)^{2}\left|y\in\mathbb{I},x_{j}\in X_{k^{*}}\right.\right\} =\mathcal{P}(\mathbf{x},\mu),
\]
is clearly $y=\mu+a$. Hence, $F(\mathbf{x}+a\mathbf{1})=F(\mathbf{x})+a$
and the LMS is shift invariant and weakly monotonic.
\end{example}
The Shorth (\citet{Andrews1972}) is the arithmetic mean of $X_{k^{*}}$
\begin{example}
\textbf{Shorth:} The shorth is given by
\[
F(\mathbf{x})=\frac{1}{h}\sum_{i=1}^{h}x_{i},\ x_{i}\in X_{k^{*}},\ h=\left\lfloor \frac{n}{2}\right\rfloor +1
\]
Since the set $X_{k^{*}}$ is unaltered under translation and the
arithmetic mean is shift invariant, then the shorth is shift invariant
and hence weakly monotonic.
\end{example}

\begin{example}
\textbf{OWA Penalty Functions}: Penalty functions having the form

\[
\mathcal{P}(\mathbf{x},y)=\sum_{i=1}^{n}w_{i}S_{i}\left(\left(\mathbf{x}-y\mathbf{1}\right)^{2}\right)
\]
define regression operators, $F(\mathbf{x})$ (\citet{Yager2010}).
Consider the following results dependent on the weight vector $\Delta=(w_{1},...,w_{n})$.\end{example}
\begin{enumerate}
\item $\Delta=\mathbf{1}$ generates Least Squares regression and $F$
is monotonic and hence weakly monotonic;
\item $\Delta=(0,...,0,1)$ generates  Chebyshev regression and $F$ is
monotonic and hence weakly monotonic;
\item Since all the terms $S_{i}\left(\left(\mathbf{x}-y\mathbf{1}\right)^{2}\right)$ are constant under transformation $(\mathbf{x},y) \to (\mathbf{x}+a\mathbf 1,y+a)$ (cf Theorem \ref{thm:shift_invariant_penalty}), the OWA regression operators are  shift-invariant for any choice of the weight vector $\Delta$.
\item For $\Delta=\begin{cases}
(0,...,0_{k-1},\nicefrac{1}{2},\nicefrac{1}{2},0,...0) & n=2k\ is\ even\\
(0,...,0_{k-1},1,0,...,0) & n=2k-1\ is\ odd
\end{cases}$ then $F$ is the Least Median of Squares operator and hence shift
invariant and weakly monotonic; and
\item For $\Delta=(1,...,1_{h},0,...,0),h=\left\lfloor \frac{n}{2}\right\rfloor +1$
then $F$ is the Least Trimmed Squares operator and hence is shift
invariant and weakly monotonic.
\end{enumerate}

In the cases 3-5 the OWA regression operators are not monotonic.

\begin{example}
\textbf{Density based means:} The density based means were introduced in \citet{Yager2013}.
Let $d_{ij}$ denote the distance between inputs $x_i$ and $x_j$.
The \emph{density based mean} is defined as
\begin{equation}\label{eq:dbm}
y=\sum_{i=1}^n w_i(\mathbf x) x_i,
\end{equation}
where
\begin{equation}\label{eq:dens}
w_i(\mathbf x)=\frac{u_i(\mathbf x)}{\sum_{j=1}^n u_j(\mathbf x)} = \frac{K_C(\frac{1}{n}\sum_{j=1}^n d_{ij}^2)}{\sum_{k=1}^n K_C(\frac{1}{n}\sum_{j=1}^n d_{kj}^2)},
\end{equation}
and where $K_C$ is Cauchy kernel given by
\begin{equation}\label{eq:cauchy}
K_C(t) = (1+t)^{-1}.
\end{equation}

As shown in \citet{Beliakov2014_INS} density based means are shift-invariant and hence weakly monotonic. Extensions of the formulas \eqref{eq:dens}, \eqref{eq:cauchy} are also presented.
\end{example}

It may appear that the class of weakly monotonic averages consists mostly of shift-invariant functions, as the above examples illustrate. This impression is due to the fact that such examples came from robust regression, where the very definition of robust estimators of location involve shift-invariance \citet{Rousseeuw1987}. However, the class of weakly monotonic functions is reacher, as various (but not all) $\varphi-$transforms of shift-invariant functions (with non-linear $\varphi$) are weakly monotonic but not shift-invariant. Some results on the conditions on $\varphi$ which preserve weak monotonicity are presented in \citet{WilkinIPMU2014}. A few more examples are presented in the sequel.

\subsection{Mixture Functions}

The mixture functions were given by Eqn. \eqref{eq:mixture_function},
which we recall here for clarity
\[
M_{w}(\mathbf{x})=\frac{{\displaystyle \sum_{i=1}^{n}}w(x_{i})x_{i}}{{\displaystyle \sum_{i=1}^{n}}w(x_{i})}.
\]

\citet{Mesiar2008} have shown that under the constraint that $w$
is non-decreasing and differentiable, if $w(x)\ge w'(x)\cdot(b-x),x\in[a,b]=\mathbb{I}$,
then $M_{w}$ is an aggregation function and hence monotonic
(and by extension, also weakly monotonic). Additionally, $M_{w}$
is invariant to scaling of the weight functions (i.e., $M_{\alpha w}=M_{w}\ \forall\alpha\in\mathbb{R}\backslash\{0\}$
). In \citet{Mesiar2006}, it was shown that the dual, $M_{w}^{d}$,
of $M_{w}$ is generated by $w(1-x)$.

As mentioned in Section \ref{sec:preliminaries}, a special case of
the Gini means (with $p=1)$ are the Lehmer means, which are generally
not monotonic. Lehmer means are mixture functions with weight function
$w(t)=t^{q}$, which is neither increasing for all $q\in\mathbb{R}$
nor shift invariant. Note that for $q<0$ the value of Lehmer means at $\mathbf{x}$ with at least one component $x_i=0$  is defined as the limit when $x_i \to 0^+$, so that $L_q$ is continuous on $[0,\infty)^n$

 We begin by establishing some general
properties of Lehmer means.
\begin{lmm}
\label{thm:Lehmer_mean_properties}The Lehmer mean $L_{q}:[0,\infty)^{n}\rightarrow[0,\infty)$,
given by

\begin{equation}
L_{q}(\mathbf{x})=\frac{{\displaystyle \sum_{i=1}^{n}}x_{i}^{q+1}}{{\displaystyle \sum_{i=1}^{n}}x_{i}^{q}},\qquad q\in\mathbb{R}\label{eq:Lehmer_mean}
\end{equation}

is
\begin{enumerate}
\item homogeneous;
\item monotonic (and linear) along the rays emanating from the origin;
\item averaging;
\item idempotent;
\item not generally monotonic in $\mathbf{x}$;
\item has neutral element $0$ for $q>0$; and,
\item has absorbing element $0$ for $q<0$.\end{enumerate}
\end{lmm}

The proof is presented in the Appendix.

We now establish a sufficient condition for weak monotonicity of Lehmer
means, which depends on both $q$ and the number of arguments $n$.
We provide a relation between these two quantities.
\begin{thm}
\label{thm:Lehmer_mean_weak_monotonicity}The Lehmer mean of n arguments,
is weakly monotonic on $[0,\infty)^{n}$ if $n\le1+\left(\frac{q+1}{q-1}\right)^{q-1}$, $ q\in\mathbb{R} \setminus (0,1)$.
\end{thm}
\begin{proof}
The Lehmer mean for $q\in[-1,0]$ is known to be monotonic (\citet{Farnsworth1986})
and hence weakly monotonic in that parameter range. In the range $q\in(0,1)$ the Lehmer mean is not weakly monotonic, because it's partial derivative at $\mathbf x=(a,b)$ when $a \to 0^+$ tends to $- \infty$.
Hence we focus
on the cases $q\geq 1$ and $q<-1$. The proof is easier to present in penalty-based representation, as the partial derivatives have more compact form.
As stated in Section \ref{thm:penalty_mixture_function},
$L_{q}(\mathbf{x})$ can be written as a penalty-based function (\ref{eq:penalty_based_function})
with penalty $\mathcal{P}(\mathbf{x},y)={\displaystyle \sum_{i=1}^{n}}x_{i}^{q}\left(x_{i}-y\right)^{2}$.
Differentiation w.r.t $y$ yields
\[
\mathcal{P}_{y}(\mathbf{x},y)=-2\sum_{i=1}^{n}\left(x_{i}^{q+1}-x_{i}^{q}y\right).
\]
At the minimum we have the implicit equation $\mathcal{P}_{y}=F(\mathbf{x},y)=0$,
with the necessary condition that yields $y=L_{q}(\mathbf{x})$.
We remind that for any $x_{i}=0$ the Lehmer mean is defined in the limit as $x_{i}\rightarrow0^{+}$.
The partial derivatives $\frac{\partial L_{q}(\mathbf{x})}{\partial x_{i}}$
are given by the implicit derivative $\frac{\partial y}{\partial x_{i}}=-\frac{F_{x_{i}}}{F_{y}}$,
with
\[
F(\mathbf{x},y)={\displaystyle \sum_{i=1}^{n}}x_{i}^{q+1}-y{\displaystyle \sum_{i=1}^{n}}x_{i}^{q}=0.
\]
By differentiation $F_{y}(\mathbf{x},y)=-{\displaystyle \sum_{i=1}^{n}}x_{i}^{q}\le0,\ \forall x_{i}\in[0,\infty)$
and thus the sign of the partial derivatives depends on the sign of
$F_{x_{i}}$, which is given by
\[
F_{x_{i}}(\mathbf{x},y)=(q+1)x_{i}^{q}-qx_{i}^{q-1}y.
\]
These derivatives can be either positive or negative. To establish
weak monotonicity we require that the directional derivative of $L_{q}(\mathbf{x})$
in the direction $(1,1,...,1)$ be non-negative. We have that $\left(\textrm{D}_{\mathbf{1}}L_{q}\right)(\mathbf{x})=\frac{1}{\sqrt{n}}\nabla L_{q}(\mathbf{x})\cdot\mathbf{1}=\frac{n}{\sqrt{n}F_{y}(\mathbf{x},y)}{\displaystyle \sum_{i=1}^{n}}F_{x_{i}}(\mathbf{x},y)$
and thus the sign of the directional derivative is determined only
by the sign of ${\displaystyle \sum_{i=1}^{n}}F_{x_{i}}(\mathbf{x},y)$.
We will henceforth work with the sorted inputs, $\mathbf{x}_{()}=\mathbf{x}_{\searrow}$
such that $\mathbf{x}_{(1)}$ is thus the largest input and $x_{(n)}$
the smallest.

Consider first the case: $q\geq 1$.

We examine the term $F_{x_{(1)}}$ and note that $y\le x_{(1)}$ for
any input $\mathbf{x}$ since $L_{q}(\mathbf{x})$ is averaging (condition
3 of Lemma \ref{thm:Lehmer_mean_properties}). Then it follows that
\begin{alignat*}{1}
F_{x_{(1)}} & =(q+1)x_{(1)}^{q}-qx_{(1)}^{q-1}y
  \ge(q+1)x_{(1)}^{q}-qx_{(1)}^{q-1}x_{(1)} =x_{(1)}^{q}\ge0.
\end{alignat*}
For the remaining $x_{i}$ we compute the smallest possible value
of $F_{x_{i}}$ by selecting the point of minimum value, which is
attained for
\[
\frac{\partial F_{x_{i}}}{\partial x_{i}}=q\left(q+1\right)x_{i}^{q-1}-q\left(q-1\right)x_{i}^{q-2}y=0.
\]
At the optimum either $x_{i}^{*}=0$ or
\begin{alignat*}{1}
q(q+1)\left(x_{i}^{*}\right)^{q-1}-q(q-1)\left(x_{i}^{*}\right)^{q-2}y & =0\\
\Rightarrow & x_{i}^{*}=\left(\frac{q-1}{q+1}\right)y \geq 0.
\end{alignat*}
At $x_{i}^{*}=0$ we have that $F_{x_{i}}=0$ (for $q>1$) and $F_{x_{i}}=-y$ (for $q=1$), and at $x_{i}^{*}=\left(\frac{q-1}{q+1}\right)y$
we have that
\begin{alignat*}{1}
F_{x_{i}}(x_{i}^{*}) & =(q+1)\left(\left(\frac{q-1}{q+1}\right)y\right)^{q}-q\left(\left(\frac{q-1}{q+1}\right)y\right)^{q-1}y\\
 & =(q-1)\left(\frac{q-1}{q+1}\right)^{q-1}y^{q}-q\left(\frac{q-1}{q+1}\right)^{q-1}y^{q}\\
 & =y^{q}\left(\frac{q-1}{q+1}\right)^{q-1}\left(q-1-q\right)\\
 & =-y^{q}\left(\frac{q-1}{q+1}\right)^{q-1}\\
 & \ge-x_{(1)}^{q}\left(\frac{q-1}{q+1}\right)^{q-1}.
\end{alignat*}
Since $(\textrm{D}_{\mathbf{1}}L_{q})(\mathbf{x})\varpropto{\displaystyle \sum_{i=1}^{n}}F_{x_{i}}$
then
\[
(\textrm{D}_{\mathbf{1}}L_{q})(\mathbf{x})=c\left(F_{x_{(1)}}+\sum_{i=2}^{n}F_{x_{(i)}}\right),
\]
and since each $F_{x_{(i)}}\ge-x_{(1)}^{q}\left(\frac{q-1}{q+1}\right)^{q-1}$
then

\begin{alignat*}{1}
(\textrm{D}_{\mathbf{1}}L_{q})(\mathbf{x}) & \ge c{\displaystyle \left(F_{x_{(1)}}+(n-1)\left(-x_{(1)}^{q}\left(\frac{q-1}{q+1}\right)^{q-1}\right)\right)}\\
 & ={\displaystyle c\left(x_{(1)}^{q}-(n-1)\left(\frac{q-1}{q+1}\right)^{q-1}x_{(1)}^{q}\right)}\\
 & =cx_{(1)}^{q}\left(1-(n-1)\left(\frac{q-1}{q+1}\right)^{q-1}\right).
\end{alignat*}
This expression is non-negative and hence $L_{q}(\mathbf{x})$ is
weakly monotonic provided that
\[
(n-1)\left(\frac{q-1}{q+1}\right)^{q-1}\le1\qquad\textrm{or}\qquad n\le1+\left(\frac{q+1}{q-1}\right)^{q-1}, q>1.
\]

For $q=1$ we get $n\leq 2$.

Now consider the case: $q<-1$.
We have that
\[
F_{x_{i}}=\frac{(1-p)x_{i}+py}{x_{i}^{p+1}},\qquad p=\left|q\right|>1
\]
and note that these derivatives are defined in the limit for the case
where $x_{i}=0$. I.e., $\left.F_{x_{i}}^{+}\right|_{x_{i}=0}={\displaystyle \lim_{x_{i}\rightarrow0^{+}}}F_{x_{i}}$.
We now examine the term $F_{x_{(n)}}$ and note that $y\ge x_{(n)}$
since $L_{q}(\mathbf{x})$ is averaging. Thus
\begin{alignat*}{1}
F_{x_{(n)}} & =\frac{(1-p)x_{(n)}+py}{x_{(n)}^{p+1}}\\
 & \ge\frac{(1-p)x_{(n)}+px_{(n)}}{x_{(n)}^{p+1}}
  =\frac{1}{x_{(n)}^{p}}.
\end{alignat*}
Again we consider the remaining $x_{i}$ by seeking the minimum of
$F_{x_{i}}$, given by
\[
\frac{\partial F_{x_{i}}}{\partial x_{i}}=-\frac{p(1-p)}{x_{i}^{p+1}}-\frac{p(p+1)}{x_{i}^{p+2}}y=0.
\]
This attains a minimum at $x_{i}=\left(\frac{p+1}{p-1}\right)y$ and
substitution into $F_{x_{i}}$ gives
\begin{alignat*}{1}
F_{x_{i}}\left(\frac{p+1}{p-1}y\right) & =\frac{(1-p)\left(\frac{p+1}{p-1}y\right)+py}{\left(\frac{p+1}{p-1}y\right)^{p+1}}\\
 & =\frac{-1}{y^{p}}\left(\frac{p+1}{p-1}\right)^{-(p+1)}\\
 & \ge\frac{-1}{x_{(n)}^{p}}\left(\frac{p+1}{p-1}\right)^{-(p+1)}.
\end{alignat*}
The directional derivative of $L_{q}(\mathbf{x})$ can be written
as
\begin{alignat*}{1}
(\textrm{D}_{\mathbf{1}}L_{q})(\mathbf{x}) & =c\left(F_{x_{(n)}}+\sum_{i=1}^{n-1}F_{x_{(i)}}\right)\\
 & \ge c\left(\frac{1}{x_{(n)}^{p}}-\frac{n-1}{x_{(n)}^{p}}\left(\frac{p+1}{p-1}\right)^{-(p+1)}\right)\\
 & =\frac{c}{x_{(n)}^{p}}\left(1-(n-1)\left(\frac{p+1}{p-1}\right)^{-(p+1)}\right).
\end{alignat*}
We note that the sign of this derivative does not change in the limit
as $x_{(n)}\rightarrow0^{+}$ and is non-negative for
\begin{alignat*}{1}
n & \le1+\left(\frac{p+1}{p-1}\right)^{p+1}\\
 & =1+\left(\frac{q-1}{q+1}\right)^{-q+1},\quad q=-p.
\end{alignat*}

Hence, in both cases $(q<-1,q\geq 1)$ we obtain the requirement for a
non-negative directional derivative - and hence weak monotonicity
of $L_{q}(\mathbf{x})$ - as being $n\le1+\left(\frac{q+1}{q-1}\right)^{q-1}$.
For the case $-1\le q\le0$ this remains a sufficient condition for
weak monotonicity, although clearly overly restrictive.
\end{proof}

\begin{rem} As suggested to us, as
$$
\left(\frac{q+1}{q-1}\right)^{q-1} = \left(\left(1+\frac{2}{q-1}\right)^{\frac{q-1}{2}}\right)^2,
$$
and the right hand side is increasing (with $q$) and approaches $e^2$ as $q \to \infty$, we have a restriction that for all $q>1$ weak monotonicity holds  for at most $n<9$ arguments. 

A further suggestion was to show that Lehmer means are weakly monotonic for any number of arguments for negative powers $q$. This can be achieved by examining the directional derivative $(\textrm{D}_{\mathbf{1}}L_{q})(\mathbf{x})$ directly, and in the near future we shall formalise this result.

So while Lehmer means are an interesting example of mixture functions (with familiar power functions as the weights) their usefulness in many applications would be limited, as they are weakly monotonic only for a restricted range of (positive) powers and the number of arguments.
\end{rem}

\begin{cor}
The contra-harmonic mean ($q=1$) is weakly monotonic only for two arguments.
\end{cor}

\subsection{Spatial-Tonal Filters} \label{sec:image}

The well known class of spatial-tonal filters includes the mode filter
(\citet{Weijer2001}), bilateral filter (\citet{Tomasi1998}) and
anisotropic diffusion (\citet{Perona1990}) among others. This is
an important class of filters developed to preserve edges within images
when performing tasks such as filtering or smoothing. While these
filters are commonly expressed in integral notation over a continuous
space, they are implemented in discrete form over a finite set of
pixels that take on finite values in a closed interval. It can be
shown that the class of functions is given (in discrete form) by the
averaging function

\begin{equation}
F_{\Delta}^{g}(\mathbf{x};x_{1})=\frac{{\displaystyle \sum_{i=1}^{n}w_{i}g(\left|x_{i}-x_{1}\right|)x_{i}}}{{\displaystyle \sum_{i=1}^{n}}w_{i}g(\left|x_{i}-x_{1}\right|)},\label{eq:bilateral_filter}
\end{equation}
where the weights $w_{i}$ are nonlinear and non-convex functions
of the locations of the pixels, which have intensity $x_{i}$. In
all practical problems the locations are constant and hence can be
pre-computed to produce the constant weight vector $\Delta=(w_{1},w_{2},...,w_{n})$.
The pixel $x_{1}$ is the pixel to be filtered/smoothed such that
its new value is $\bar{x}_{1}=F_{\Delta}^{g}(\mathbf{x};x_{1})$.

The function $F_{\Delta}^{g}$ is nonlinear and not monotonic. It
is trivially shown to be expressed as a penalty-based function with
penalty
\[
\mathcal{P}(\mathbf{x},y)=\sum_{i=1}^{n}w_{i}g(\left|x_{i}-x_{1}\right|)(x_{i}-y)^{2}.
\]
In image filtering applications it is known that this penalty minimises
the mean squared error between the filtered image and the noisy source
image (\citet{Elad2002}). By Proposition \ref{thm:shift_invariant_penalty}
it follows directly that the filter $F_{\Delta}^{g}$ is shift invariant
and hence weakly monotonic. Furthermore, Theorem \ref{thm:shift_invariant_penalty}
permits us to generalise this class of filters to be those penalty
based averaging functions having penalty function

\begin{equation}
\mathcal{P}(\mathbf{x},y)=\sum_{i=1}^{n}w_{i}g(\left|x_{i}-f(\mathbf{x})\right|)(x_{i}-y)^{2}\label{eq:penalty_filter}
\end{equation}
or even further using other  bivariate function $D:\mathbb{I}^{2}\rightarrow\mathbb{R}$
(as discussed in Remark \ref{remark:penalty_dissimilarity_function})

\begin{equation}
\mathcal{P}(\mathbf{x},y)=\sum_{i=1}^{n}w_{i}g(\left|x_{i}-f(\mathbf{x})\right|)D(x_{i},y)\label{eq:penalty_filter-1}
\end{equation}

The implication of replacing $x_{1}$ with $f(\mathbf{x})$ in the
scaling function $g$ is that we may use any shift-invariant aggregation
of $\mathbf{x}$, which allows us to account for the possibility that
$x_{1}$ is itself an outlier within the local region of the image. For example, we could use the median, the mode or the shorth for $f(\mathbf x)$.
This provides an interesting result and invites further research in
the application of weakly monotonic means to spatial-tonal filtering
and smoothing problems.

\section{Conclusion}

\label{sec:conclusions}

In this article we have introduced the concept of weakly monotonic
averaging functions and examined some of the properties of these
functions. We have studied several families of means previously considered
to be simply non-monotonic, and shown them to be weakly monotonic.
Specifically we have established a sufficient condition for the weak
monotonicity of the Lehmer mean - which is an important subclass of
the Mean of Bajraktarevic - and shown that several important non-monotonic
regression operators are actually weakly monotonic. Additionally we
have proven that a large class of penalty-based  functions
are also weakly monotonic, which admits a very large class of aggregation
functions. This has permitted a simple proof that the class of image
processing filters known as spatial-tonal filters are weakly monotonic
averaging functions. This class subsumes the class of spatial averaging
filters, such as the Gaussian blur filter. Most importantly, given
the definition of weak monotonicity, all  aggregation functions
are weakly monotonic and thus we have not needed to redefine monotonic
aggregation in order to relate it to weakly monotonic averaging.

This study was prompted by two issues. First, that there exist several
important classes of means that fall outside of the current definition
of aggregation functions, which requires monotonicity in all arguments.
These include the examples presented in Section \ref{sec:wm_examples}:
the robust estimators of location (such as the mode and the shorth),
mixture functions and the spatial-tonal filters used extensively in
image processing. It appears reasonable to treat these functions within
the same framework that includes the monotonic means.

The second issue is that applications such as image processing and
robust statistics require non-monotone averaging, where the main concern
is that of noise (outliers) within the data. While the average must
be a representative value of the inputs, we wish to avoid the possibility
that one or more erroneous inputs drives the value of the output.
A small increase above the average may be reasonable, however a large
increase should permit that input to be discounted or ignored and
the average to possibly decrease.

The concept of weakly monotone aggregation addresses both of these
issues, bringing the existing (monotone) aggregation functions and
many of the non-monotone means into the same framework. Our proposal
then is to redefine the class of averaging aggregation functions to be not those
functions that are monotonic, but rather the class of weakly monotonic
functions that are averaging.

\smallskip{}

\textbf{Proposal:} \emph{}{A function $F:\mathbb{I}^{n}\rightarrow\mathbb{I}$
is an averaging aggregation function on $\mathbb{I}^{n}$ if and only if it
is weakly monotonic non-decreasing on $\mathbb{I}$ and averaging.}

\smallskip{}

It remains to be seen whether or not weak monotonicity is the minimal
requirement for defining averaging aggregation functions and whether
this weaker definition is justified for other types of aggregation,
such as conjunctive and disjunctive functions. Furthermore, are there
other possibilities for the relaxation of monotonicity that provide
for a unified framework of aggregation theory and practice? We leave
these questions for future work.

\section*{References\textmd{\normalsize \bibliographystyle{elsarticle-harv}
\bibliography{weakmon}
}}

\section*{Appendix}

Proof of Lemma \ref{thm:Lehmer_mean_properties}.

\begin{proof}
Consider each of the following:
\begin{enumerate}
\item \textbf{Homogeneous}: Set $\mathbf{x}=\lambda\mathbf{u}$ then
\[
L_{q}(\mathbf{x})=L_{q}(\lambda\mathbf{u})=\frac{{\displaystyle \sum_{i=1}^{n}}(\lambda u_{i})^{q+1}}{{\displaystyle \sum_{i=1}^{n}}(\lambda u_{i})^{q}}=\lambda\frac{{\displaystyle \sum_{i=1}^{n}}u_{i}^{q+1}}{{\displaystyle \sum_{i=1}^{n}}u_{i}^{q}}=\lambda L_{q}(\mathbf{u}).
\]
Hence $L_{q}$ is homogeneous with degree $1$.
\item \textbf{Monotonic (and linear) along the rays}: Consider the generalised
spherical coordinates (\citet{Blumenson60}) $(r,\theta,\phi_{1},...,\phi_{n-2})$,
$r\ge0$, $0\le\theta\le2\pi$, $0\le\phi_{i}\le\pi$ for the hypersphere
$S^{n-1}=\{\mathbf{x}\in\mathbb{R}^{n}:\left\Vert \mathbf{x}\right\Vert =r\}$.
We will restrict the angle variables so that $x_{i}\in[0,\infty)$.
The transformation to an orthonormal Euclidean basis $E_{n}$ produces
the vector $\mathbf{x}$ of length $r$ having components
\begin{alignat*}{1}
x_{1} & =r\cos(\phi_{1})\\
x_{j} & =r\cos\left(\phi_{j}\right){\displaystyle \prod_{k=1}^{j-1}\sin\left(\phi_{k}\right)},\quad j=2,...,n-1\\
x_{n} & =r\prod_{k=1}^{n-1}\sin\left(\phi_{k}\right)
\end{alignat*}
where $\phi_{n-1}=\theta$ and $0\le\phi_{i}\le\nicefrac{\pi}{2},\ i=1,...,n-1,\ r\ge0$.
The Lehmer mean of the Euclidean vector $\mathbf{x}$ is therefore
\begin{alignat*}{1}
L_{q}(\mathbf{x}) & =\frac{\left[r\cos(\phi_{1})\right]^{q+1}+{\displaystyle \sum_{j=2}^{n-1}}\left[r\cos(\phi_{j}){\displaystyle \prod_{k=1}^{j-1}}\sin(\phi_{k})\right]^{q+1}+\left[r{\displaystyle \prod_{k=1}^{n-1}}\sin(\phi_{k})\right]^{q+1}}{\left[r\cos(\phi_{1})\right]^{q}+{\displaystyle \sum_{j=2}^{n-1}}\left[r\cos(\phi_{j}){\displaystyle \prod_{k=1}^{j-1}}\sin(\phi_{k})\right]^{q}+\left[r{\displaystyle \prod_{k=1}^{n-1}}\sin(\phi_{k})\right]^{q}}\\
 & =r\left[\frac{\left[\cos(\phi_{1})\right]^{q+1}+{\displaystyle \sum_{j=2}^{n-1}}\left[\cos(\phi_{j}){\displaystyle \prod_{k=1}^{j-1}}\sin(\phi_{k})\right]^{q+1}+\left[{\displaystyle \prod_{k=1}^{n-1}}\sin(\phi_{k})\right]^{q+1}}{\left[\cos(\phi_{1})\right]^{q}+{\displaystyle \sum_{j=2}^{n-1}}\left[\cos(\phi_{j}){\displaystyle \prod_{k=1}^{j-1}}\sin(\phi_{k})\right]^{q}+\left[{\displaystyle \prod_{k=1}^{n-1}}\sin(\phi_{k})\right]^{q}}\right]\\
 & =rf(\phi_{1},...,\phi_{n-1})
\end{alignat*}
Along rays emanating from the origin each $f(\cdot)$ is constant
and hence $L_{q}(\mathbf{x})=\alpha_{\phi}r$ is linear.
\item \textbf{Averaging}: Let $\mathbf{x}_{\sigma}=\mathbf{x}_{\searrow}$
and take $a=x_{(1)}$ and $b=x_{(m)}$ denote the value of the largest
and the smallest non-zero elements of $\mathbf{x}$ respectively.
By homogeneity
\[
L_{q}(\mathbf{x})=a\frac{1+{\displaystyle \sum_{i=2}^{n}}\left(\frac{x_{(i)}}{a}\right)^{q+1}}{1+{\displaystyle \sum_{i=2}^{n}}\left(\frac{x_{(i)}}{a}\right)^{q}}=a\frac{1+\alpha}{1+\beta}
\]
Since $x_{(i)}\le a$ for all $i=1,...,n$ then $\alpha\le\beta$.
Hence $L_{q}(\mathbf{x})\le a$. Similarly,
\[
L_{q}(\mathbf{x})=b\frac{1+{\displaystyle \sum_{i=1}^{m-1}}\left(\frac{x_{(i)}}{b}\right)^{q+1}}{1+{\displaystyle \sum_{i=1}^{m-1}}\left(\frac{x_{(i)}}{b}\right)^{q}}=b\frac{1+\gamma}{1+\delta}
\]
Since $x_{(i)}\ge b$ for all $i=1,...,m-1$ and $x_{(i)}=0$ for
all $i=m+1,...,n$ then $\delta\le\gamma$. Hence $L_{q}(\mathbf{x})\ge b$.
Thus, $\min(\mathbf{x})\le L_{q}(\mathbf{x})\le\max(\mathbf{x})$
and $L_{q}(\mathbf{x})$ is averaging.
\item \textbf{Idempotent}: For any vector $\mathbf{x}=(t,t,...,t)$ we have
that
\[
L_{q}(\mathbf{x})=\frac{t{\displaystyle \sum_{i=1}^{n}}t^{q}}{{\displaystyle \sum_{i=1}^{n}}t^{q}}=t
\]
and hence $L_{q}$ is idempotent.
\item \textbf{Not generally monotonic in x}: Take $\mathbf{x}=(1,0)$ and
$\mathbf{y}=(1,\nicefrac{1}{2})$, then for $q>0$, $L_{q}(\mathbf{x})=1$
and $L_{q}(\mathbf{y})=\frac{1+\left(\nicefrac{1}{2}\right)^{q+1}}{1+\left(\nicefrac{1}{2}\right)^{q}}=
\frac{\left(2^{q+1}+1\right)}{\left(2^{q+1}+2\right)}<1$.
Thus $\mathbf{x}<\mathbf{y}$ and $L_{q}(\mathbf{x})>L_{q}(\mathbf{y})$,
hence $L_{q}(\mathbf{x})$ is not generally monotonic in \textbf{x}
for all $q\in\mathbb{R}$.
\item \textbf{Has neutral element of 0} \textbf{for $q>0$}: Consider $\mathbf{x}=(a,0)$
then $L_{q}(\mathbf{x})={\displaystyle \lim_{x_{2}\rightarrow0^{+}}}\frac{a^{q+1}+{\displaystyle x_{2}^{q+1}}}{a^{q}+x_{2}^{q}}=a$
for $q>0$.
\item \textbf{Has absorbing element of $0$ for $q<0$}: Consider $\mathbf{x}=(a,0)$
then $L_{q}(\mathbf{x})={\displaystyle \lim_{x_{2}\rightarrow0^{+}}L_{q}(1,x_{2})=}\lim_{x_{2}\rightarrow0^{+}}\frac{a+{\displaystyle x_{2}^{q+1}}}{a+x_{2}^{q}}=\frac{\frac{a}{x_{2}^{q}}+x_{2}}{\frac{a}{x_{2}^{q}}+1}=0$.
\end{enumerate}
\end{proof}

\end{document}